\documentclass{article}
\usepackage{graphicx} 

\usepackage[T1]{fontenc}
\usepackage[latin9]{inputenc}
\usepackage{amsmath}
\usepackage{amssymb}
\usepackage{mathtools}
\usepackage{amsthm}

\usepackage{appendix}
\usepackage{hyperref}

\usepackage{algorithm}

\usepackage{algpseudocode}
\algrenewcommand\algorithmicrequire{\textbf{Parameters:}}
\newcommand{\eqdef}{\stackrel{\mathrm{def}}{=}}

\usepackage{subcaption}

\newtheorem{proposition}{Proposition}[section]
\newtheorem{lemma}{Lemma}[section]

\newtheorem{theorem}{Theorem}[section]

\newtheorem{definition}{Definition}[section]
\newtheorem{remark}{Remark}[section]
\newtheorem{observation}{Observation}[section]

\usepackage{natbib}

\usepackage{fullpage}
\usepackage{graphicx}
\graphicspath{{figures/}}
\usepackage{bbm}
\usepackage[]{color-edits} 
\usepackage{subcaption}
\usepackage{thm-restate}

\newcommand{\ind
}[1]{\mathbbm{1}_{#1}}

\newcommand{\reals}{\mathbb{R}}
\newcommand{\ALG}{\mathrm{ALG}}
\newcommand{\E}{\mathbb{E}}

\usepackage{tikz}

\usepackage[nice]{nicefrac}

\usepackage{orcidlink}

\newcommand{\prm}[1]{PM(#1)}
\addauthor{TE}{magenta}

\addauthor{TG}{red}

\newcommand{\supp}{\textbf{Supp}}
\newcommand{\supi}[1]{\supp_{#1}}

\title{The Competition Complexity of Prophet Inequalities with Correlations\footnote{T.~Ezra is supported by the Harvard University Center of Mathematical Sciences and Applications. T.~Garbuz is funded by the European Research Council (ERC) under the European Union's Horizon 2020 research and innovation program (grant agreement No. 101077862).}}
\author{Tomer Ezra\thanks{Harvard University, Cambridge, USA. Email: \texttt{tomer@cmsa.fas.harvard.edu}. \orcidlink{0000-0003-0626-4851}} 
\and Tamar Garbuz\thanks{Tel Aviv University, Tel Aviv, Israel. Email: \texttt{garbuz@mail.tau.ac.il}. \orcidlink{0009-0009-3922-4160}}}

\date{}

\begin{document}

\maketitle
\begin{abstract}

We initiate the study of the prophet inequality problem through the resource augmentation framework in scenarios when the values of the
rewards are correlated. Our goal is to determine the number of additional rewards an online algorithm requires to approximate the maximum value of the original instance. 
While the independent reward case is well understood, we extend this research to account for correlations among rewards. 
Our results demonstrate that, unlike in the independent case, the required number of additional rewards for approximation depends on the number of original rewards, and that block-threshold algorithms, which are optimal in the independent case, may require an infinite number of additional rewards when correlations are present. 
We develop asymptotically optimal algorithms for the following three scenarios: (1) where rewards arrive in blocks corresponding to the different copies of the original instance; (2) where rewards across all copies are arbitrarily shuffled; and (3) where rewards arrive in blocks corresponding to the different copies of the original instance, and values within each block are pairwise independent rather than fully correlated.

\end{abstract}

\newpage
\pagenumbering{arabic}

\section{Introduction}
In the classical prophet inequality setting \citep{krengel1977semiamarts,krengel1978semiamarts}, a decision-maker observes a sequence of $n$ rewards online one by one, and upon observing each, she needs to decide immediately and irrevocably whether to accept the current reward which terminates the process, or discard it and continue to the next reward.
The rewards are distributed according to some known distribution, and the goal of the decision-maker is to maximize her expected chosen reward.
Traditionally, the decision-maker's performance is measured using the competitive-ratio measure, which is the ratio between the expectation of the chosen value of the decision-maker and the expected value of a ``prophet'' that can observe values yet to come, and always selects the maximum value.

The seminal work of \citet{krengel1977semiamarts,krengel1978semiamarts} established that for every product distribution of the rewards (i.e., when the values of the rewards are independent), the decision-maker can always guarantee a competitive-ratio of $\nicefrac{1}{2}$, and no better competitive-ratio can be guaranteed for all distributions.
Moreover, \citet{samuel1984comparison} showed that this competitive-ratio can be achieved by a simple single-threshold algorithm (an algorithm that selects the first reward with a value exceeding a predetermined threshold calculated from the outset).
This structural property of prophet inequality made it very appealing for sequential pricing, as the result of \citet{samuel1984comparison} implies that the optimal competitive-ratio of sequential pricing auction can be achieved by a single anonymous static price.

Recent research has integrated the concept of resource augmentation into the prophet inequality framework. Specifically, \citet{DBLP:conf/sigecom/BrustleCDV22} introduced the notion of competition complexity within this context. This concept assesses how many additional rewards (or additional instances of the original prophet inequality) a given online algorithm requires to closely match the prophet's performance. The goal is to identify the number of additional instances needed for an algorithm to achieve 
$1-\epsilon$ of the expected value of the prophet, thereby measuring the algorithm's performance across varying levels of $\epsilon$. They found that while the exact competition complexity (where $\epsilon=0$) is unbounded, for 
$\epsilon>0$, it exhibits a double logarithmic growth in  
$\nicefrac{1}{\epsilon}$ in i.i.d. settings.

The appeal of competition complexity stems from the straightforward nature of online algorithms, which do not necessitate coordination among decisions. This simplicity broadens their applicability in various contexts. For example, consumers are generally more inclined to respond to simple, take-it-or-leave-it sale offers, than to engage in intricate, coordinated auctions that require simultaneous bidding of all participants.

In a follow-up work, \citet{DBLP:journals/corr/abs-2402-11084} explored the efficacy of block-threshold algorithms in non-i.i.d settings, showing that the result of \cite{DBLP:conf/sigecom/BrustleCDV22} extend to independent but not identical distributions.
Moreover, they established structural properties of the optimal online algorithms, and showed that block-threshold algorithms (algorithms that only adjust the threshold at the beginning of each copy) achieve the optimal competitive-ratio which grows double logarithmically in $\nicefrac{1}{\epsilon}$.

The result of \cite{DBLP:journals/corr/abs-2402-11084} has opened avenues for examining more complex scenarios, such as combinatorial prophet inequalities, which involve a decision-maker that can select multiple rewards according to some feasibility constraint.

Our work extends these ideas to the correlated prophet inequality problem, tackling scenarios where there is a dependency between the rewards.
Unlike classical settings (where independence is assumed), no online algorithm can achieve a competitive-ratio better than $\nicefrac{1}{n}$ \citep{hill1983stop}. The introduction of competition complexity in correlated settings breaks this barrier.

This shift is crucial, as it mirrors numerous practical situations where understanding and leveraging correlations between events can significantly enhance decision-making efficacy.

Consider for example a seller that wants to sell an umbrella and wants to give it to the one who values it the most. Each day a sequence of potential buyers consider buying the umbrella from the seller each with his own private value. The seller can only sell the umbrella when each buyer arrives in an online fashion. On rainy days, it is very likely that the value of most of the potential buyers is higher (and thus, correlated). Our guiding question in this paper for this scenario is: ``How many days are necessary for the seller while selling online to approximate well the expected maximum welfare of a single day?''

\subsection{Our Contribution}

In this paper, we consider the competition complexity of correlated prophet inequality under two natural arrival models, namely, the block and adversarial arrival models.
Under the block arrival model, the $kn$ rewards arrive according to their blocks, with all correlated rewards within a block arriving sequentially before any rewards from subsequent blocks.
Under the adversarial arrival model, the $kn$ rewards arrive in an arbitrary order.
Our main contribution can be partitioned into two parts, structural and quantitative.

\paragraph{Structural Contribution.}
We first observe that in contrast to previous works that considered the independent case (\citep{DBLP:conf/sigecom/BrustleCDV22} and \citep{DBLP:journals/corr/abs-2402-11084}), the competition complexity of the correlated prophet inequality depends on the number of rewards in the original instance $n$ (see Proposition~\ref{prop:hard1}). 

Our second structural insight is that the structural properties of the optimal strategy observed in \cite{DBLP:journals/corr/abs-2402-11084} do not carry over to the correlated prophet inequality case.
In particular, we show (see Proposition~\ref{prop:hard2}) that no block-threshold algorithm can achieve a finite competition complexity, even under the case where $n=2$, and under the block arrival model.

\paragraph{Quantitative Contribution.}
Our quantitative results can be summarized in the following table:
\begin{table}[h!]
\begin{center}
\begin{tabular}{ |c |c |c |}
\hline
  & Independent Prophet Inequality & Correlated Prophet Inequality \\
 \hline
 Block Arrival Model & $\Theta\left( \log\log\left( \nicefrac{1}{\epsilon}\right)\right)$ \citep{DBLP:journals/corr/abs-2402-11084} & $\mathbf{\Theta}\left( \boldsymbol{n}+ \boldsymbol{\log}\boldsymbol{\log}\left( \boldsymbol{\nicefrac{\mathbf{1}}{\mathbf{\epsilon}}}\right)\right)$ \\  
 \hline
 Adversarial Arrival Model & $\Theta\left(  \nicefrac{1}{\epsilon}\right)$ \citep{DBLP:journals/corr/abs-2402-11084} & $\mathbf{\Theta}\left(  \boldsymbol{\nicefrac{n}{\epsilon}}\right)$ \\ 
 \hline   
\end{tabular}
\caption{Competition complexity of prophet inequality: previous and new results. Our new results are indicated in bold. Our upper bound for the block arrival model is established in Theorem~\ref{thm:cc} and the lower bounds are established in Propositions~\ref{prop:hard1} and \ref{prop:tight}. Our result for the adversarial arrival model is established in Theorem~\ref{thm:cc-adversary}.} 
\label{table:results}
\end{center}
\end{table}

Our results establish that while the factor $n$, the number of rewards in the original instance, impacts the competition complexity additively in the block arrival model, it impacts the competition complexity multiplicatively in the adversarial arrival model. 

Our result for the block arrival model is not only asymptotically tight but also achieves tightness in the leading coefficient constants when aiming to approximately stochastically dominate the prophet (see discussion in Section~\ref{sec:hardness}).

All of our results are constructive, where upper bounds are devised by implementable efficient algorithms, and all of our new lower bounds are given by explicit hard instances.

\paragraph{Pairwise Independence.} Our last result is an adaptation of our algorithm and analysis for the correlated prophet inequality to the pairwise independent case.
In particular, we devise an algorithm with a simplified analysis that achieves the optimal asymptotic competition complexity of $O\left(\log\log(\nicefrac{1}{\epsilon})\right)$.
This pairwise independent case generalizes the main setting considered in \cite{DBLP:journals/corr/abs-2402-11084} for independent but not identical distributions. Moreover, our simplified analysis allows us to achieve a better multiplicative constant. In particular, our multiplicative constant is $\frac{\log(2)}{\log(\nicefrac{5}{4})} \approx 3.1$ better than the multiplicative constant of \citep{DBLP:journals/corr/abs-2402-11084} and our algorithm applies to more generalized settings.

 \subsection{Our Techniques}
 \paragraph{Block Arrival Model.}
We first use the construction of hard instances of \citet{hill1983stop} for correlated prophet inequality in order to devise a lower bound on the competition complexity for the correlated prophet inequality. 
We do so by bounding the expected value of a decision-maker on an instance with $k$ copies by at most $k$ times the expected value on a single copy.

     Then, we construct (Proposition~\ref{prop:hard2}) a hard instance, showing that single-threshold (and also block-threshold) algorithms have unbounded competition complexity. This instance is devised by two correlated random variables, where the first variable is $\frac{1}{\epsilon}^{G(\epsilon)}$ where $G(\epsilon)$ is a geometric random variable with a success probability $\epsilon$, and the second variable is always $\frac{1}{\epsilon}$ times the first variable. Any threshold set by the decision-maker would almost surely select the first variable (if selecting any) leading to a low expectation. This instance captures the difficulty of using thresholds in correlated instances compared to independent instances.

      In order to devise our main positive result of an upper bound of $O\left(n + \log\log\left(\nicefrac{1}{\epsilon}\right)\right)$ for the block arrival model, we analyze the distribution of the maximum value (see Figure~\ref{fig:prophet}). 
      \begin{figure}
          \centering
          \includegraphics[width=0.47\linewidth]{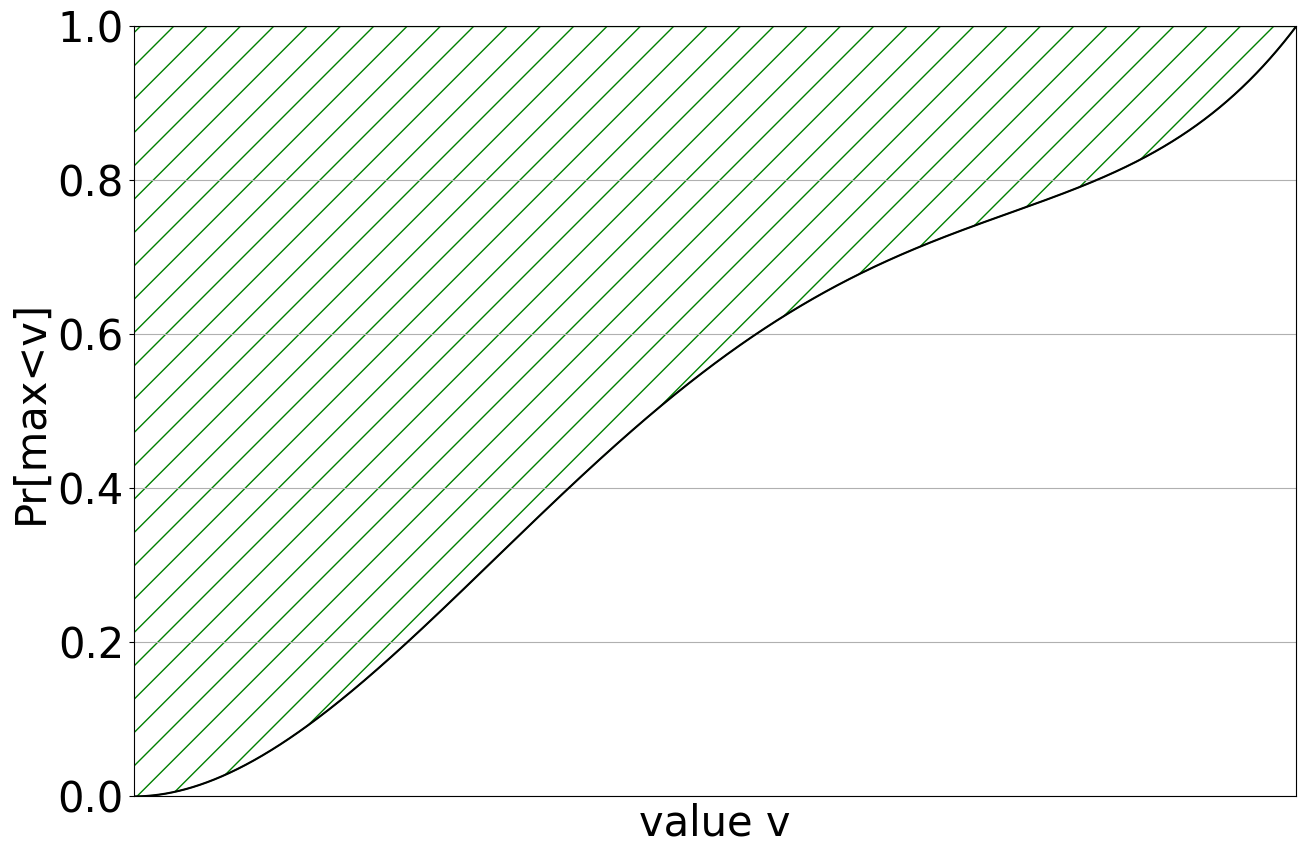}
          \caption{In this figure, the black line represents the CDF of the maximum value. The green area above it is the expected value of the prophet.}
          \label{fig:prophet}
      \end{figure}
     
     We first partition the expectation of the maximum, to values above and below the $(1-1/n)$-quantile. Our algorithm then works in two phases, where the first phase fully covers the contribution of values above the $(1-1/n)$-quantile, followed by the second phase, which covers the contribution of values below the $(1-1/n)$-quantile up to the lowest $\epsilon$-quantile.

     For values above the $(1-1/n)$-quantile, we carefully adapt the approach of Online Contention Resolution Schemes (OCRS) \citep{feldman2016online} for the independent prophet inequality into correlated settings. OCRS is a rounding method for online optimization problems.  
     The natural OCRS algorithm for the independent prophet inequality selects each value with a probability that is proportional to the probability of it being the maximum. This guarantees that if the algorithm selects a value, then it is distributed the same as the distribution of the maximum value. The only issue is that it selects a value with a probability that is strictly smaller than 1 (which is also known as the selectability factor). 
     Our first adaptation to the natural OCRS is to handle correlations which cause us to have a selectability factor of $\frac{1}{n}$ (in contrast to the traditional selectability factor of $\frac{1}{2}$ in the independent case).
     Our second adaptation is to apply a threshold at the $(1-1/n)$-quantile value (on top of the OCRS) which allows us to fully cover contribution above the $(1-1/n)$-quantile after applying our OCRS for $n+1$ copies of the original instance (see Figure~\ref{fig:p1}). If we did not have the added threshold, the algorithm would only cover a fraction of the contribution that approaches $1$ exponentially which only gives an exponential (rather than double exponential) convergence.
     The advantage of this approach is that if we select a value during these iterations, it is distributed the same as the distribution of the maximum value given that it is above the $(1-1/n)$-quantile. The disadvantage of this approach is that the limitation of OCRS for correlated values is that it cannot select with a probability of more than $\frac{1}{n}$, which means that if we want to select some value with probability at least $1-\epsilon$ (which is a necessary condition for getting the desired competition complexity), we must apply this OCRS for at least $k=\Omega\left(n\cdot \log(1/\epsilon)\right)$ copies overall. In other words, this type of algorithm does not select with a high enough probability at each iteration to be applied for all rounds to obtain the desired competition complexity.

    Therefore, after applying the OCRS for $O(n)$ iterations, we switch our algorithm to doubly exponentially decreasing quantile thresholds. Since we already covered the full contribution of values above the $(1-1/n)$-quantile of the maximum, setting a threshold as the $(1-1/n)$-quantile covers the values between the $(1-1/n)$-quantile and the $(1-1/n)^2$-quantile of the maximum. After applying the $(1-1/n)$-quantile as a threshold once, we can now decrease the threshold to the $(1-1/n)^2$-quantile which covers the value between the $(1-1/n)^2$-quantile, and the $(1-1/n)^4$-quantile.
    After each such round, we can square the quantile of the next threshold which covers all values above the $\epsilon$-quantile after $O\left(\log(n) + \log\log(\frac{1}{\epsilon})\right)$ (see Figure~\ref{fig:p2}).
    This leads to an overall competition complexity of $O\left(n+\log\log(\frac{1}{\epsilon})\right)$.
     
\begin{figure}
\begin{subfigure}{.47\textwidth}
  \centering
  \includegraphics[width=1\linewidth]{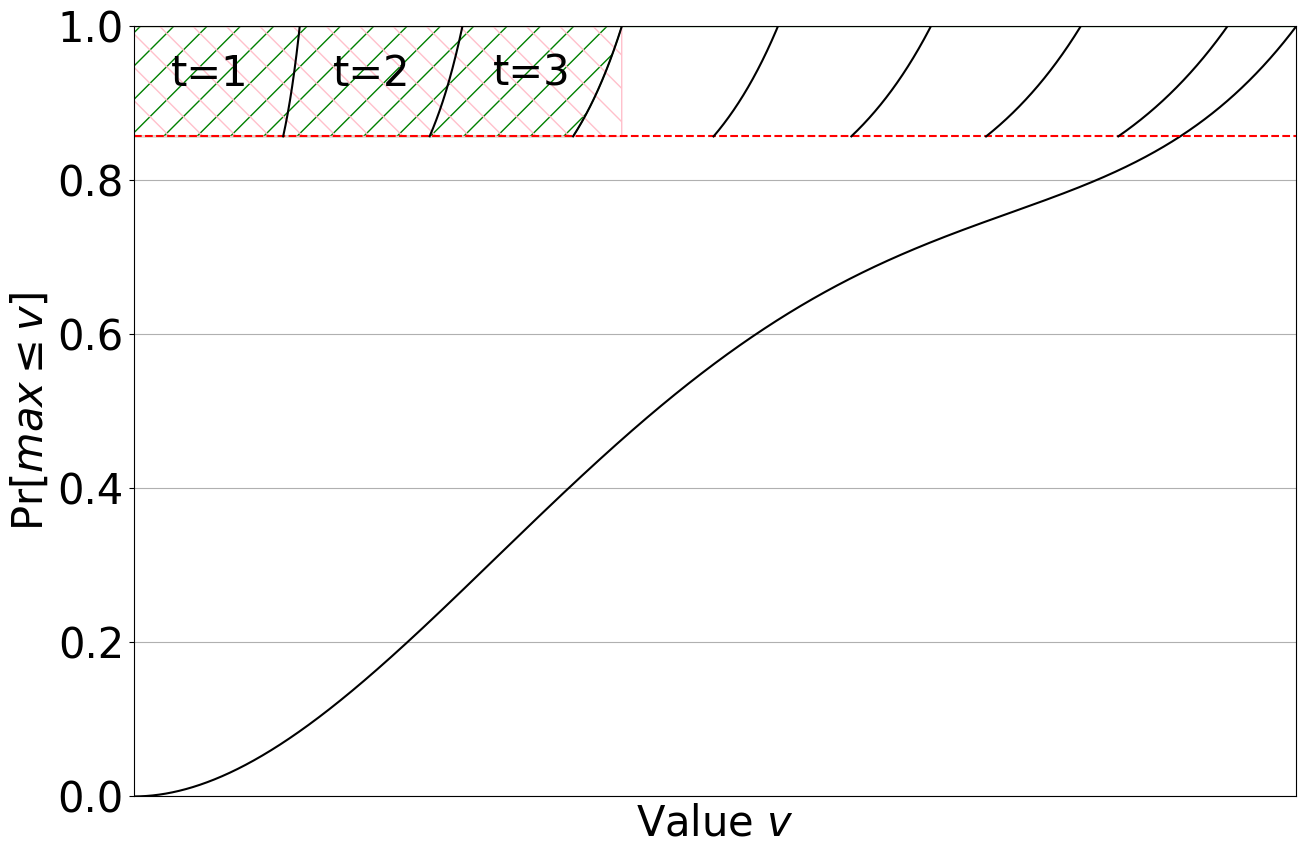}
  \caption{Phase 1: In this phase, the algorithm only selects values above the $(1-1/n)$-quantile. Here we demonstrate the expected value and the stopping probability at the end of the third round of this phase.\hphantom{xxxxxxxxxxxxxxxxxx}}
  \label{fig:p1}
\end{subfigure}
\hspace{0.05\textwidth} 
\begin{subfigure}{.47\textwidth}
  \centering
  \includegraphics[width=1\linewidth]{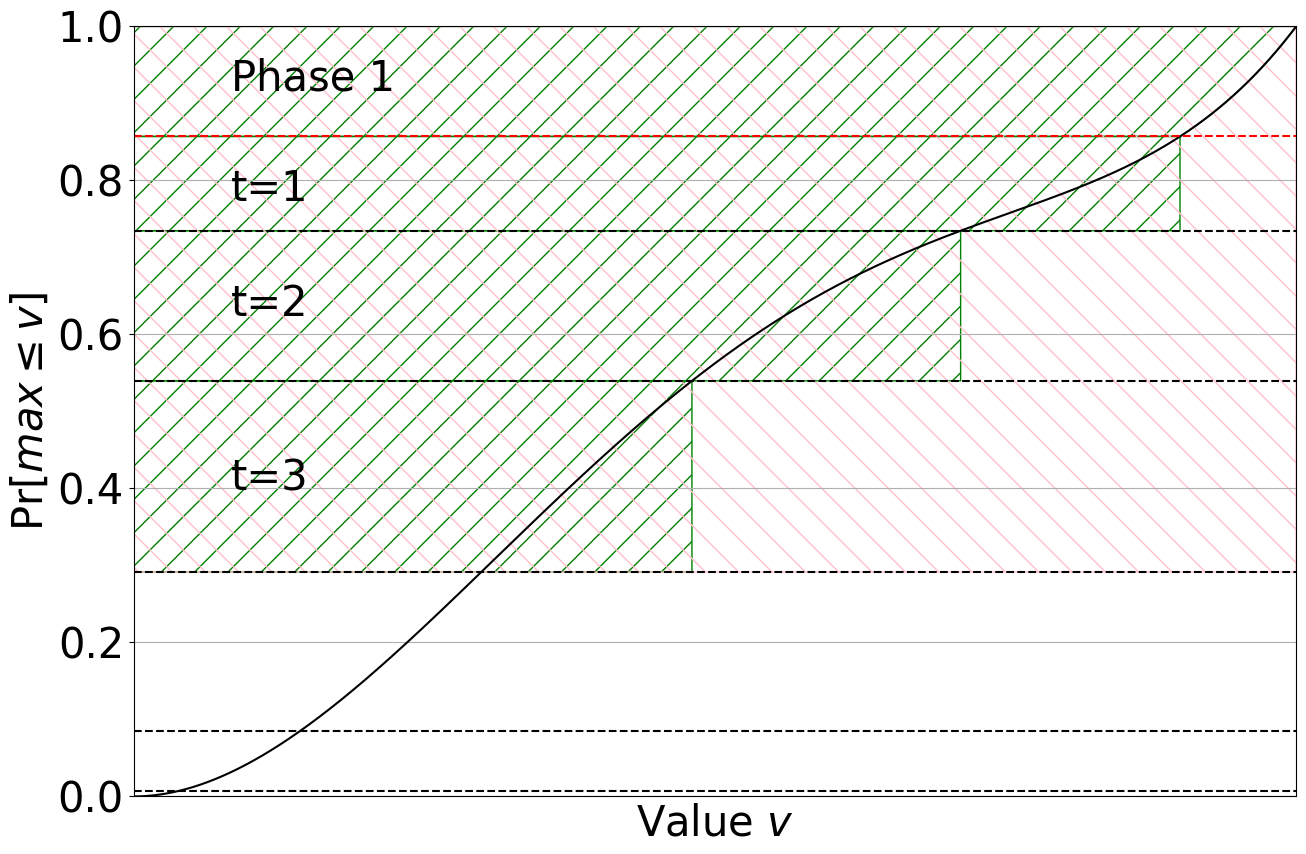}
  \caption{Phase 2: In this phase, the algorithm selects values above different thresholds (which are captured by the black dotted lines). Here we demonstrate the expected value and the stopping probability at the end of the third round of this phase.}
  \label{fig:p2}
\end{subfigure}
\caption{The accumulated value of our algorithm throughout the phases. The red dotted line is the $(1-1/n)$-quantile. The green area is the expected value of the algorithm up to the end of the round of each phase. The pink area (over the total area) is the probability that the algorithm selects some value up to the end of the round of the phase.}
\label{fig:alg}
\end{figure}

\paragraph{Adversarial Arrival Model.}
For the adversarial arrival model, our main technique for devising our algorithm is calculating the expected value of each reward above the threshold $T$ which we define as $1-\epsilon$ times the value of the prophet. By a union-bound type argument, we know that there must be a reward for which this expected value is at least $\nicefrac{\epsilon}{n}$ times the value of the prophet. Thus, applying a threshold algorithm for independent copies of this reward (while ignoring other rewards) allows us to use the surplus-revenue analysis of the prophet inequality \citep{KleinbergW19} which gives us an upper bound on the competition complexity of $O\left(\nicefrac{n}{\epsilon}\right)$.

We devise a matching lower bound by carefully combining the hard instance (Proposition~\ref{prop:hard1}) with an additional deterministic reward. The additional value reduces the contribution of all the randomized rewards to approximately $\nicefrac{\epsilon}{n}$ times the value of the prophet. Moreover, our arrival order is such that any algorithm can either accept the deterministic reward (which is not high enough to obtain $(1-\epsilon)$-fraction of the prophet) or obtain at most the contribution of one reward from each copy. Thus, to get $(1-\epsilon)$-fraction of the prophet, the algorithm must use at least $\Omega\left(\nicefrac{n}{\epsilon}\right)$ copies.

\paragraph{Pairwise Independence.}
Our last result, of adapting the analysis of the block arrival model to the case where the rewards are pairwise independent, relies on improving the guarantees of the first phase. We use several observations on threshold algorithms from \citep{caragiannis2021relaxing} to show that a single constant-quantile threshold can serve a similar purpose as our OCRS for the fully correlated case.
This allows us to both improve the number of rounds needed for the first phase from $n+1$ to $2$, and allows us to decrease the quantile of the threshold substantially to be a constant-quantile which removes the $\log(n)$ term in the number of rounds needed for the second phase.

\subsection{Further Related Work}
The prophet inequality captures many 
real-life scenarios such as hiring candidates, setting dynamic pricing for a sequence of potential buyers, etc.
The correlated prophet inequality was first studied in \citep{rinott1991orderings,rinott1992optimal,rinott1987comparisons} where they showed that the guarantees of the prophet inequality continues to hold under negative correaltions. Beyond negative correlations, \citet{immorlica2020prophet} considered the case of linear correlations.
Correlated prophet inequalities were used as a tool in \citep{cai2021simple} for the case where the distributions'  dependency is represented by a Markov Random Field. \citet{caragiannis2021relaxing} considered the case of correlated prophet inequality where the distributions are pairwise independent. \citet{dughmi2023limitations} extended the correlated prophet inequality framework to matroid feasibility constraints, and devised optimal guarantees for matroids that satisfy the partition property as well as hardness results for general matroids (both apply to pairwise independence distributions).

It has been studied extensively through several dimensions, such as extending to combinatorial settings, comparing to the best online instead of the prophet, different arrival models, and limited information settings. 

\paragraph{Combinatorial Settings.} The prophet inequality has been first extended to combinatorial setting in \citep{kennedy1985optimal,kertz1986comparison,kennedy1987prophet}  that studied the multiple-choice variant of the prophet inequality, 
matroid feasibility constraints~\citep{azar2014prophet,KleinbergW19}, 
polymatroids~\citep{dutting2015polymatroid}, matching~\citep{GravinW19,ezra2022prophet}, combinatorial auctions \citep{feldman2014combinatorial,dutting2020prophet,dutting2020log}, and downward-closed (and beyond) feasibility constraints~\citep{DBLP:conf/stoc/Rubinstein16}.

\paragraph{Different Arrival Models.} The prophet inequality problem has been extensively studied in  different arrival models beyond the standard worst-case order such as the random order arrival (also known as the prophet secretary) \citep{esfandiari2017prophet,azar2018prophet,ehsani2018prophet,CorreaSZ21},
and the free-order arrival model where the decision-maker can choose the order of distributions in which rewards are drawn from \citep{beyhaghi2018improved,AgrawalSZ20,PengT22}.

\paragraph{Alternative Benchmarks.} Our work suggests a different metric to measure the performance of online algorithms through resource augmentation. The prophet inequality framework has also been studied for many additional alternative benchmarks (instead of the traditional competitive-ratio) such as comparing to the best online algorithm (instead of comparing to the prophet) \citep{niazadeh2018prophet,papadimitriou2021online,saberi2021greedy,kessel2022stationary,braverman2022max,ezra2023next,srinivasan2023online,ezra2023importance,ezra2024choosing}, measuring the expected ratio to the prophet instead of the ratio of expectations \citep{ezra2023prophet}, and allowing the decision-maker to replace its chosen value up to a certain amount of times \citep{ezra2018prophets}.

\section{Model}

\paragraph{Correlated Prophet Inequality.} Consider the problem where a decision-maker faces a sequence $v =(v_1,\ldots,v_n)$ of $n$ rewards that arrive online and needs to select one of them in an online fashion (not selecting a value is equivalent to receiving a reward of $0$). The sequence of rewards $v$ is distributed according to some distribution $F\in \Delta_n $ where $\Delta_n$ is the set of all distributions over sequences in $\reals_{\geq 0 }^n$. Note that the distribution $F$  is over the entire sequence and allows correlations between different rewards of the sequence. We refer to the case where the rewards are independent as the independent prophet inequality. 
Traditionally, the performance of the decision-maker is measured using the competitive-ratio benchmark, which is the ratio between the expected value chosen by the decision-maker (that makes decisions online) and the expected value of the maximum in hindsight. We denote an online algorithm for the decision-maker by $\ALG$, and the (possibly random) value chosen by $\ALG$ on a realized sequence $v$ by $\ALG(v)$.
Thus, the competitive-ratio of an algorithm $\ALG$ is $$ \rho(\ALG) = \inf_{F \in  \Delta_n} \frac{\E_{v\sim F}[\ALG(v)]}{\E_{v\sim F}[\max_{i} v_i]},$$
where the expectation of the numerator might also depend on the randomness of the algorithm if such exists.

\paragraph{Competition Complexity.}  Our goal is to compare the expected performance of an online algorithm on $k$ independent copies of the original instance, to the expected maximum value of a single instance,  and ask how many copies are needed for achieving a certain approximation to the value of the prophet. 
Formally, consider an algorithm $\ALG$ that observes a sequence of $n k$ values partitioned into $k$ rounds $v^{(1)},\ldots,v^{(k)}$, where for each round $t \in [k] \eqdef \{1,\ldots,k\}$, $v^{(t)} = (v^{(t)}_1,\ldots,v^{(t)}_n)$ is sampled independently from $F$. The algorithm $\ALG$ observes the $n k$ values one by one, and needs to choose one of them in an online fashion.
The goal is to design an algorithm with as few as possible copies that approximates the expected maximum value of a single copy.
We measure it by the competition complexity measure.

\begin{definition}[Competition Complexity]\label{def:competition} Given $\varepsilon \ge 0$, the $(1-\varepsilon)$-competition complexity of algorithm $\ALG$ is the smallest positive integer  $k(n,\varepsilon)$ such that for  every $F\in \Delta_n$, and every $k\ge k(n,\varepsilon)$, it holds that
\[
\E_{v\sim F^k}[\ALG(v)] \geq (1-\varepsilon) \cdot \E_{v \sim F}[\max_{i} v_i].
\]
\end{definition}

The case $\varepsilon = 0$ is also referred to as \emph{exact} competition complexity; it was shown in 
\cite[Theorem 2.1]{DBLP:conf/sigecom/BrustleCDV22} that the exact competition complexity is unbounded even for the i.i.d.~case (namely, where $F=F' \times \ldots \times F' $ for some distribution $F' \in \Delta_1$). So, naturally, our focus will be on the case $\varepsilon > 0$.

Given a vector $v=(v_1,\ldots,v_n)$ and a parameter $i \leq n$ we denote by $v_{\leq i}$ the prefix of length $i$ of $v$, i.e., $v_{\leq i} =  (v_1,\ldots,v_i)$.
Given a distribution $F$, and a vector $x=(x_1,\ldots,x_i)$ of length $i \leq n$, we denote by $F_x$ the distribution of $v$ conditioned on $v_{\leq i} =x$. 
We also use the notation $\prm{x}$ to represent the probability that, for a prefix $x$ of length $i\leq n$, the element $x_i$ is the maximum value within the entire sequence. This probability is computed based on the randomness of future reward realizations, and is formulated by $\prm{x} = \Pr_{v \sim F_x}[x_i = \max_j v_j]$.

For simplicity of presentation, we assume throughout the paper that the distribution $F$ is discrete, and we denote its support by $\supp$. We also denote by $\supi{\leq i}$ the set of all the prefixes of length $i\leq n$ of vectors in $\supp$. For a vector $x \in \supi{\leq i}$ we define $f(x) = \Pr_{v \sim F}[v_{\leq i}=x]$. 
Moreover, we assume that for every realization $v\in \supp$, the maximum value is unique, i.e., $|\{j \mid v_j = \max_i v_i\}|=1$.
These assumptions are made to simplify the analysis, and our algorithms can be adjusted easily to hold with respect to non-discrete distributions or when the maximum is not unique using standard techniques from the literature.

We denote by $F^*$  the distribution of $\max_{i} v_i$ where $v\sim F$.
When using $\log$, the base of the logarithm is $2$ unless stated otherwise. When writing $\Pr[I]$, for an indicator $I$, we mean  $\Pr[I=1]$. If $I$ was not explicitly defined, it means that $I=0$. For a $q\in [0,1]$ we use $Ber(q)$ to denote a Bernoulli random variable with probability $q$, and $G(q)$ to denote a geometric random variable with a success probability of $q$.

\paragraph{Online Contention Resolution Schemes (OCRS).} Part of our algorithm is based on the OCRS approach. An OCRS solves the following selection problem.
An OCRS algorithm works on a ground set of elements, in which a random subset of it is \textit{active}. The algorithm observes the elements one by one, and upon observing each, the element's activeness is revealed to the algorithm, and the algorithm needs to decide whether to select it or not.
The goal of the algorithm is to overall select a feasible subset of the active elements (among a predefined feasibility constraint), such that only active elements are selected and each active element is selected with the same probability\footnote{In the definition of OCRS by \citet{feldman2016online} they defined it such that the algorithm needs to select with a probability of at least $c$, but for our analysis, we require that it will be exactly $c$.} $c$ which is called the \textit{selectability factor}.  

\paragraph{Quantile Thresholds and Stochastic Domination.} Part of the analysis of our algorithm uses thresholds. A $q$-quantile threshold with respect to distribution $D$ is a value $T$ such that $\Pr_{v \sim D}[v < T] \leq q \leq  \Pr_{v \sim D}[v \leq  T]$.
Given two distributions $D_1,D_2$, we say that $D_1$ stochastically dominates $D_2$ if for every value $x$ it holds that $\Pr_{v\sim D_1}[ v \geq x] \geq \Pr_{v\sim D_2}[ v \geq x] $. We say that $D_1$ $c$-stochastically dominates $D_2$ for $c<1$, if  for every value $x$ it holds that $\Pr_{v\sim D_1}[ v \geq x] \geq c\cdot \Pr_{v\sim D_2}[ v \geq x] $. We say that $D_1$ stochastically dominates $D_2$ up to $\epsilon$-quantile if for every value $x$, either $\Pr_{v\sim D_1}[ v \geq x] \geq \Pr_{v\sim D_2}[ v \geq x] $ or $\Pr_{v\sim D_2}[ v < x] \leq \epsilon$. It holds that if $D_1$ stochastically dominates $D_2$ up to $\epsilon$-quantile then it also $(1-\epsilon)$-stochastically dominates $D_2$. If the distribution of the value returned by an algorithm $\ALG$ stochastically dominates $F^*$ up to $\epsilon$-quantile, then $\ALG$ satisfies the requirement in Definition~\ref{def:competition}. 
\begin{observation}\label{obs:quantile}
To show that the competition complexity $k(n,\epsilon) \leq k'$ it suffices to show that for every distribution $F\in \Delta_n$ there exists an online algorithm $\ALG$ that observes $k'$ rounds of values such that the distribution of the value returned by $\ALG$ stochastically dominates $F^*$ up to $\epsilon$-quantile. 
\end{observation}

\paragraph{Additional Models.} Beyond the repeated block model, we consider the following cases:
\begin{itemize}
\item Adversarial order: In this model, the $n\cdot k$ rewards arrive in an adversarial order (i.e., not partitioned into $k$ blocks).
\item  Pairwise independent prophet inequality: In this model, the rewards of each block are pairwise independence. In particular, this generalizes both the i.i.d. and the independent (but not identical) cases considered in \citep{DBLP:conf/sigecom/BrustleCDV22} and \citep{DBLP:journals/corr/abs-2402-11084}, respectively.
\end{itemize}

\section{Competition Complexity of Correlated Prophet Inequality}\label{sec:ccc}

In this section, we present an algorithm with optimal competition complexity for the correlated prophet inequality.

Our algorithm has two phases consisting of $m_1,m_2$ rounds respectively.
The first phase of the algorithm is implemented in Algorithm~\ref{alg:cap-p1}.
\begin{algorithm}
\caption{Competition Complexity of Correlated Prophet Inequality - Phase 1}\label{alg:cap-p1}
\begin{algorithmic}[1]

\Require{$m_1$ - number of rounds, $T_0$ - a threshold}

\For{ round $t=1,\ldots , m_1   $}
\State Let $A_{t,1} = 1$
\For { $i = 1,\ldots,n$} 
\State Sample $Y_{t,i} \sim Ber(\prm{v_{\leq i}^{(t)}}) $
\State Sample $Z_{t,i}\sim Ber\left(\frac{1}{n  - \sum_{j \in [i-1]}  \prm{v^{(t)}_{\leq j}}}\right)$
\If{$Y_{t,i}=Z_{t,i}=A_{t,i}=1$}
\State $A_{t,i+1} = 0$
\If{ $v_i^{(t)}>T_0$}
\State  Return $v_i^{(t)}$ 
\EndIf
\Else
\State $A_{t,i+1} = A_{t,i}$
\EndIf
\EndFor
\EndFor
\end{algorithmic}
\end{algorithm}

Algorithm~\ref{alg:cap-p1} applies the same procedure for all of the $m_1$ rounds it runs. In each such round, it applies an OCRS for one copy of the correlated prophet inequality instance with a selectability factor of $\frac{1}{n}$. In this OCRS, the ground set of elements is the $n$ rewards, and the feasibility constraint is to select at most one element.
The OCRS at round $t$ is implemented by the random variables of $\{Y_{t,i}, Z_{t,i}\}_{i}$ that determine the activeness of the rewards and which of them to select. Variable  $Y_{t,i}=1$ means that reward $i$ is active for round $t$ (and $Y_{t,i}=0$ or undefined means that $i$ is not active for this round). It is equivalent to sampling the future arrivals of the copy given the current history and $i$ is active if it is the maximum among the history of the current copy and the fresh sample of rewards that didn't arrive yet at this round.
Variables $Z_{t,i}$ decide how the OCRS resolves contentions between the different rewards. If $Z_{t,i}=1$ it means that the OCRS prioritized selecting reward $i$.
The variable $A_{t,i}=0$ means that an active reward was prioritized before reward $i$ in round $t$.
A reward is selected if it is the first one of the round that is both active and prioritized, and is above the threshold $T_0$.

The second phase is defined in Algorithm~\ref{alg:cap-p2}.
\begin{algorithm}
\caption{Competition Complexity of Correlated Prophet Inequality - Phase 2}\label{alg:cap-p2}
\begin{algorithmic}[1]
\Require{$m_2$ - number of rounds, $T_0$ - an initial threshold}
\State Let $p = \Pr[\max_j v_j < T_0 ]$ \label{step:p}
\State Let $T_i $ be a $  p^{(2^i)}$-quantile threshold with respect to distribution $ F^*$
\For{ round $t=1,\ldots , m_2  $}
\State Let $B_{t}=1$
\For{$i=1,\ldots, n$}
\If{$v_i^{(t)} \geq T_{\max\{t -2,0\}}$}
\State Return $v_i^{(t)}$
\EndIf
\EndFor
\EndFor
\end{algorithmic}
\end{algorithm}

The second phase of the algorithm sets a series of quantile thresholds for doubly exponentially decreasing quantiles. Recall that $p$ (defined in Step~\ref{step:p} of Algorithm~\ref{alg:cap-p2}) is the probability that the maximum in a single round is strictly below the first threshold. Our algorithm's thresholds guarantee that at the end of round $t$, the probability that no reward is chosen is $p^{2^{t-1}}$. 
The indicator $B_t$ defines whether we reach round $t$ (if $B_t$ is not defined, we treat it as $B_t=0$ which means that we didn't reach round $t$). One difference from the description in the introduction (for simplicity of presenting it without the formal model) is that we are using the threshold $T_0$ twice at the beginning of the second phase (instead of only once in the description in the introduction). The goal of using it twice is to handle cases where $T_0$ lies on a point-mass, and since in the first phase we only cover values that are strictly above $T_0$, this step is necessary for the analysis. If there are no point masses, this step can be avoided.

Our overall algorithm is captured in Algorithm~\ref{alg:cap}.
\begin{algorithm}
\caption{Competition Complexity of Correlated Prophet Inequality}\label{alg:cap}
\begin{algorithmic}[1]
\Require{$\epsilon$ - target approximation}

\State Let $T_0$ be a $ (1- \frac{1}{n})  $-quantile threshold with respect to distribution $F^*$
\State Let $m_1=n+1$ 
\State Let $ m_2 = \lceil \log\log\frac{1}{\varepsilon} + \log n \rceil +2 $
\State Apply Algorithm~\ref{alg:cap-p1} with Parameters $m_1$ and $T_0$ \Comment{Phase 1}
\State If Algorithm~\ref{alg:cap-p1} didn't return a value, apply Algorithm~\ref{alg:cap-p2} with parameters $m_2$ and $T_0$ \Comment{Phase 2}
\end{algorithmic}
\end{algorithm}

The algorithm runs first Algorithm~\ref{alg:cap-p1} for $m_1$ iterations with $T_0$ defined as the $\left(1-\frac{1}{n}\right)$-quantile threshold. This phase covers the contribution of values above $T_0$. If the algorithm didn't select a reward in the first phase, then in the second phase, which runs Algorithm~\ref{alg:cap-p2} with $m_2$, and the same value of $T_0$, the algorithm applies a series of thresholds in which it covers the contribution of values below $T_0$, where after each round $t$ of Algorithm~\ref{alg:cap-p2}, combined with the contribution of the first phase, it overall covers the contribution of values above $T_{t-1}$.

\begin{theorem}\label{thm:cc}
Algorithm~\ref{alg:cap} achieves a competition complexity of $(1+o(1)) \cdot (n +\log\log\frac{1}{\varepsilon})$.    
\end{theorem}

   For simplicity, we denote the value of  $\lceil \log\log\frac{1}{\varepsilon} + \log n \rceil  $ by $r$. 
Our algorithm runs for $m_1+m_2 = n +r+3 $ rounds, i.e., the competition complexity is $k(n,\varepsilon) = n+r+3$. We also define the variable  $B_{r+3}$ that denotes whether the algorithm arrived to the end without returning a value. We set $B_{r+3}$ to be $1$ if the algorithm didn't return a value.   
To prove Theorem~\ref{thm:cc}, we first prove that Algorithm~\ref{alg:cap} is well defined.
In particular, we need to show that 
the probabilities that define $Y_{t,i}$ and $Z_{t,i}$ are indeed between $0$ and $1$ for all $t,i$.
This is true since for every prefix $x$, $\prm{x} \in [0,1]$, and therefore it holds that $\sum_{j \in [i-1]} \prm{v^{(t)}_{\leq j}}  \in [0,n-1]$, and thus, $$ \frac{1}{n} \leq \frac{1}{n  - \sum_{j \in [i-1]}  \prm{v^{(t)}_{\leq j}}} \leq  \frac{1}{n-(n-1)} = 1, $$
which implies that the algorithm is well-defined.

\paragraph{Approximation.} By Observation~\ref{obs:quantile} it is sufficient to show that for every $y\geq T_r$ it holds that \begin{equation} \label{eq:sd}
    \Pr_{v\sim F^k}[\ALG(v) \geq y] \geq \Pr_{v\sim F}[\max_i v_i  \geq y].
\end{equation} 
This is sufficient since for $y< T_r$ it holds that $$\Pr[ \max_i v_i < y]  \leq \Pr[ \max_i v_i < T_{r}]  \leq p^{2^r} \leq \left(1-\frac{1}{n}\right)^{2^r} \leq  \varepsilon,$$
where $p$ is as defined in Algorithm~\ref{alg:cap-p2}, which by definition of $T_0$ satisfies that $p \leq 1-\frac{1}{n}$.

We define the following variables:
\begin{eqnarray}
\label{eq:q}
    q  & = & \Pr[\max_i v_i > T_0] =  \sum_{i=1}^n\sum_{x\in \supi{\leq i}} f(x)\cdot \prm{x}  \cdot \ind{x_i > T_0}, 
    \mbox{ and }\\
        q_y & = & \Pr[\max_i v_i \geq y] =  \sum_{i=1}^n\sum_{x\in \supi{\leq i}} f(x)\cdot \prm{x}  \cdot \ind{x_i \geq  y}.\label{eq:qx}
\end{eqnarray}
We partition the proof of Inequality~\eqref{eq:sd} into three cases, depending on the value of $y$.

\paragraph{Case 1: $y>T_0$.}
Note that since $y>T_0$ it holds that $q_y \leq q$.
We first show that we can assume that $q>0$ as otherwise, Inequality~\eqref{eq:sd} holds trivially for values of $y>T_0$ since $$    \Pr_{v \sim F^k}[\ALG(v) \geq y] \geq 0=q = \Pr_{v\sim F}[\max_i v_i  > T_0]  \geq  \Pr_{v\sim F}[\max_i v_i  \geq y]   .$$
For this case, we analyze the first phase of the algorithm. We next prove the following property for the event of reaching reward $i$ at round $t$. 
\begin{lemma}
The probability of reaching reward $i$ at round $t$ (of Algorithm~\ref{alg:cap-p1}), given that we reached round $t$  for $t\leq m_1$, with history $v^{(t)}_{\leq i-1} = x$, for some $x\in \supi{\leq i-1}$ is the following expression:
\begin{equation}\label{eq:ati}
    \Pr[A_{t,i}\mid A_{t,1} \wedge v^{(t)}_{\leq i-1} =x]  = 1- \sum_{j<i} \frac{\prm{v^{(t)}_{\leq j}}}{n}.
    \end{equation} \label{lem:reach}
\end{lemma}
\begin{proof}
We prove the lemma by induction on $i$.
The base case where $i=1$ holds trivially, as both sides of the equation are equal to $1$.
Assume that the equation holds whenever $j \leq  i$, now \begin{eqnarray*}
 \Pr[A_{t,i+1}\mid A_{t,1} \wedge v^{(t)}_{\leq i} =x] & = &  \Pr[A_{t,i+1}\mid A_{t,i} \wedge v^{(t)}_{\leq i} =x] \cdot \Pr[A_{t,i}\mid A_{t,1} \wedge v^{(t)}_{\leq i} =x]  \\
 &= & (1-\Pr[Y_{t,i} \wedge Z_{t,i} \mid  v^{(t)}_{\leq i} =x ]) \cdot \left(1- \sum_{j<i} \frac{\prm{v^{(t)}_{\leq j}}}{n}\right) \\ 
 & = & \left(1- \prm{v^{(t)}_{\leq i}} \cdot \left(\frac{1}{n-\sum_{j<i} \prm{v^{(t)}_{\leq j}}} \right) \right)\cdot \left(1- \sum_{j<i} \frac{\prm{v^{(t)}_{\leq j}}}{n}\right) \\ 
 & = & 1- \sum_{j \leq i} \frac{\prm{v^{(t)}_{\leq j}}}{n},
 \end{eqnarray*}
 where the first equality is since $A_{t,i+1}=1$ can only happen if $A_{t,i}=1$, and when conditioning on $A_{t,i}$, the event $A_{t,1}$ is redundant and can be omitted; the second equality is by the definition of the algorithm, by induction, and since $\Pr[A_{t,i}\mid A_{t,1} \wedge v^{(t)}_{\leq i} =x] = \Pr[A_{t,i}\mid A_{t,1} \wedge v^{(t)}_{\leq i-1} =x_{\leq i-1}]$; the third equality is by the definitions of $Y_{t,i}$, and $Z_{t,i}$.  
 This concludes the proof of the lemma.
\end{proof}

We next calculate the probability that Algorithm~\ref{alg:cap-p1} stops at each round.
\begin{lemma}
    The probability that Algorithm~\ref{alg:cap-p1} returns a value at each round $t \leq m_1$ given it arrived to round $t$ is exactly $\frac{q}{n}$. \label{lem:stop}
\end{lemma} 

\begin{proof}
    For each of those rounds, $A_{t,1}$ represents the event that we reached round $t$. 
    \begin{eqnarray}
\Pr[\neg  A_{t+1,1} \mid A_{t,1}] & = & \sum_{i=1}^n \sum_{x \in \supi{\leq i} } f(x) \cdot \Pr[A_{t,i} \cdot Y_{t,i} \cdot Z_{t,i} \mid A_{t,1}  \wedge v^{(t)}_{\leq i} = x ] \cdot \ind{x_i > T_0}       \nonumber \\ 
& \stackrel{\eqref{eq:ati}}{=} & \sum_{i=1}^n \sum_{x \in \supi{\leq i} } f(x) \left(1- \sum_{j < i} \frac{PM(v^{(t)}_{\leq j})}{n} \right) \cdot \prm{x} \cdot \left(\frac{1}{n-\sum_{j<i} \prm{v^{(t)}_{\leq j}}} \right)\cdot \ind{x_i > T_0}      \nonumber \\ 
& = &  \frac{1}{n}  \sum_{i=1}^n \sum_{x \in \supi{\leq i} } f(x)  \cdot \prm{x} \cdot \ind{x_i > T_0}      \stackrel{\eqref{eq:q}}{=} \frac{q}{n}.\label{eq:arrivet}
\end{eqnarray}
\end{proof}

We next show that if the algorithm stops at the first phase, then the probability that the returned value is at least $y$ is $\frac{q_y}{q}$.
\begin{lemma} \label{lem:qy}
For every $t\leq m_1$ it holds that:
\begin{equation}
    \Pr[\ALG(v)  \geq y \mid  \neg A_{t+1,1}  \wedge A_{t,1}] = \frac{q_y}{q}. \label{eq:t}
\end{equation}
\end{lemma}
\begin{proof}    
It holds that:
\begin{eqnarray}
\Pr[\ALG(v) \geq y \mid  \neg A_{t+1,1}  \wedge A_{t,1}]  & = & \frac{\E[\ind{\ALG(v)\geq y}   \cdot \ind{\neg A_{t+1,1}}  \mid A_{t,1}]}{\Pr[\neg A_{t+1,1}  \mid A_{t,1}]} \nonumber \\ 
& \stackrel{\eqref{eq:arrivet}}{=} &  \frac{n}{q} \cdot \E[\ind{\ALG(v)\geq y}   \cdot \ind{\neg A_{t+1,1}}  \mid A_{t,1}] \nonumber \\ 
& = & \frac{n}{q}\sum_{i=1}^n \sum_{x \in \supi{\leq i} } f(x) \cdot \Pr[A_{t,i} \cdot Y_{t,i} \cdot Z_{t,i} \mid A_{t,1}  \wedge v^{(t)}_{\leq i} = x ] \cdot \ind{x_i \geq y}    \nonumber \\
& =& \frac{1}{q}  \sum_{i=1}^n \sum_{x \in \supi{\leq i} } f(x)  \cdot \prm{x} \cdot \ind{x_i \geq y}  \stackrel{\eqref{eq:qx}}{=} \frac{q_y}{q}, \nonumber
\end{eqnarray}
where the first equality is by definition of conditional probability; the third equality is by the law of total probability dividing into all cases where the algorithm selects reward $i$ given prefix $x$ up to $i$ at round $t$; the fourth equality is since the event of $A_{t,i},Y_{t,i},Z_{t,i}$ are independent given $A_{t,1},v_{\leq i }^{(t)}$ and by Equation~\eqref{eq:ati} and the definition of $Z_{t,i}$, $\Pr[A_{t,i}] \cdot Pr[Z_{t,i}] =\frac{1}{n}$, and $\Pr[Y_{t,i}] = \prm{x}$.
\end{proof}

We next prove Inequality~\eqref{eq:sd} for $y>T_0$. 
\begin{eqnarray*}
       \Pr_{v\sim F^k}[\ALG(v) \geq y] &\geq  &  \sum_{t=1}^{m_1}  \Pr[A_{t,1}] \cdot \Pr[\neg A_{t+1,1}   \mid A_{t,1}] \cdot \Pr[\ALG(v) \geq y \mid  \neg A_{t+1,1}  \wedge A_{t,1}]  \\ 
       & = & \sum_{t=1}^{m_1} \left(1-\frac{q}{n}\right)^{t-1} \cdot \frac{q}{n} \cdot \frac{q_y}{q} \\ 
       & = &  \frac{1-\left(1-\frac{q}{n}\right)^{m_1}}{1- \left(1-\frac{q}{n}\right)}  \cdot \frac{q_y}{n} \\ 
       & \geq & q_y =   \Pr_{v\sim F}[\max_i v_i  \geq y],
\end{eqnarray*}
where the first inequality is since we only count the cases where the algorithm stops at the first phase and returns a value of at least $y$; the first equality is by Lemmas~\ref{lem:stop} and \ref{lem:qy}; the second equality is by a geometric sum; the second inequality holds since $m_1=n+1$ and $q\leq \frac{1}{n}$.

\paragraph{Case 2: $T_r < y \leq T_0$.}

We next bound the probability that  Algorithm~\ref{alg:cap} reaches each round  of Algorithm~\ref{alg:cap-p2}.
\begin{lemma}\label{lem:bt1}
For  $t=3,\ldots,m_2$ it holds that   \begin{equation}
     \Pr[B_{t}=1]  \leq    \Pr[\max_i v_i \leq T_{t-2}]. \label{eq:a2n}
\end{equation}
\end{lemma} 
\begin{proof}
     
Let $$\ell_t = \begin{cases}
     \max \{i \leq t-3 \mid T_i=T_{i-1}\} & \text{if } \{i \leq t-3 \mid T_i=T_{i-1}\} \neq \emptyset \\
    0 & \text{otherwise}
\end{cases}.$$ 
Then, it holds that:
\begin{eqnarray}    
    \Pr[B_{t}=1]   & =&   \underbrace{\left(1-\frac{q}{n}\right)^{m_1}}_{\text{Phase 1 }}  \cdot \underbrace{\Pr[\max_i v_i <T_{0}]}_{\text{Round } 1} \cdot  \prod_{u=2}^{t-1} \underbrace{\Pr[\max_i v_i <T_{u-2}]}_{\text{Round } u}    \nonumber \\  & \leq  & \Pr[\max_{i} v_i <T_{\ell_t}]^2 \cdot   \prod_{u=\ell_t+3}^{t-1} \Pr[\max_{i} v_i < T_{{u-2}}] \nonumber
      \\ & \leq  & \left(p^{-2^{\ell_t}}\right)^2 \cdot   \prod_{u=\ell_t+3}^{t-1}\left(p^{-2^{u-2}}\right) 
       = p^{-2^{t-2}} \leq    \Pr[\max_i v_i \leq T_{t-2}], \nonumber
\end{eqnarray}
where the first equality is by partitioning into the rounds of the first phase which by Lemma~\ref{lem:stop} each round stops with a probability of $\frac{q}{n}$, and to rounds of the second phase which use threshold $T_0$ at round 1, and in rounds $u=2,\ldots,t-1$ use threshold $T_{u-2}$; the first inequality is by removing some of the probabilities, and since threshold $T_{\ell_t}$ appears at least twice; the second inequality is by definition of the thresholds; the second equality is since $2\cdot 2^{\ell_t}+ \sum_{u=\ell_t+3}^{t-1} 2^{u-2} = 2^{t-2}$;
the last inequality is again by the definition of the thresholds. 
\end{proof}

Let $t^*= \max \{ t \mid T_t \geq y\}$. Note that $t^*$ is well defined since $y\leq T_0$, and it holds that $t^* < r$ since $y>T_r$.
We are now ready to prove Inequality~\eqref{eq:sd} for this case:
\begin{eqnarray*}
      \Pr_{v\sim F^k}[\ALG(v) \geq y] &\geq  & 1- \Pr[B_{t*+3}] \geq 1-  \Pr[\max_i v_i \leq T_{t^*+1}] \geq 1-  \Pr[\max_i v_i < y] =  \Pr_{v\sim F}[\max_i v_i  \geq y],
\end{eqnarray*}
where the first inequality is since we only select values above $y$ before round $t^*+3$; the second inequality is by Lemma~\ref{lem:bt1} which we can use since $t^*+3 \leq m_2$; the third inequality holds as by definition of $t^*$, we have that $T_{t^*+1}<y$.

\paragraph{Case 3: $y=T_r$.}
Our algorithm only selects values that are at least $T_r$. The probability of selecting a value is at least the probability of selecting a value on the last round which is exactly $q_y$ which concludes the proof. 

Theorem~\ref{thm:cc} then holds since we proved Inequality~\eqref{eq:sd} for every value of $y\geq T_r$, and since $\Pr_{v\sim F}[\max_i v_i  < T_r] \leq \varepsilon$.

\section{The Competition Complexity under Adversarial Orders}

In this section, we analyze settings beyond the block model, where the $n\cdot k$ rewards may arrive in an adversarial order.
Our main result in this section is a tight competition complexity for the adversarial arrival order.

\begin{theorem}
The $(1-\epsilon)$-competition complexity for the adversarial arrival order is $\Theta\left(\nicefrac{n}{\epsilon}\right)$. \label{thm:cc-adversary}
\end{theorem}

We next present our algorithm for the adversarial order case (see Algorithm~\ref{alg:cap-adv}).

\begin{algorithm}
\caption{Competition Complexity of Correlated Prophet Inequality - Adversarial Order}\label{alg:cap-adv}
\begin{algorithmic}[1]
\Require{$\epsilon$ - target approximation}

\State Let $ OPT=\E_{v\sim F}[\max_{i} v_i]$ \State For every $i$, let $R_i= E_{v\sim F} [(v_i-(1-\epsilon)OPT)^+]$ 
\State Let $i^* =\arg\max_{i} R_i$
\State Select the first reward that is of type $i^*$, and has a value of at least $(1-\epsilon)OPT$
\end{algorithmic}
\end{algorithm}

We next bound the competition complexity of Algorithm~\ref{alg:cap-adv}.
\begin{lemma}\label{lem:adver-pos}
    For every instance with $n$ rewards defined by  distribution $F$, for every $k\geq \frac{n}{\epsilon}$, Algorithm~\ref{alg:cap-adv} on  $n\cdot k$ rewards guarantees for every arrival order that $$ \E[\ALG(v)] \geq  (1-\epsilon)\cdot \E[\max_i v_i].$$
\end{lemma}
\begin{proof}
    We first show that $R_{i^*} \geq \frac{\epsilon\cdot OPT}{n}$.
    This follows since 
    \begin{eqnarray}    
    R_{i^*}  & \geq & \frac{1}{n}  \sum_{i=1}^n R_i = \frac{1}{n}\sum_{i=1}^n  \E_{v_i\sim F_i} [(v_i-(1-\epsilon)OPT)^+] =\frac{1}{n}  \E_{v\sim F} \left[\sum_{i=1}^n(v_i-(1-\epsilon)OPT)^+\right] \nonumber \\ & \geq  & \frac{1}{n} \E_{v\sim F}  \left[\max_{i}(v_i-(1-\epsilon)OPT)^+\right] \geq  \frac{1}{n} \E_{v\sim F}  \left[\max_{i}(v_i-(1-\epsilon)OPT)\right] \nonumber \\ & = & \frac{1}{n}[(OPT-(1-\epsilon)OPT)]   =  \frac{\epsilon\cdot OPT}{n}. \label{eq:ri}
    \end{eqnarray}
    Next, we analyze the algorithm's performance. Since the algorithm only selects among rewards of type $i^*$, and the different copies of rewards of type $i^*$ have independent values, we can follow the analysis of \cite{KleinbergW19}. Let $p=\Pr[\ALG \mbox{ selects a value}]$. Then 
    \begin{eqnarray}
        \E[\ALG(v)] & = & p\cdot (1-\epsilon)OPT + \sum_{j=1}^k \Pr[\ALG \mbox{ reached the $j{\mbox{th}}$ of } i^*] \cdot \E_{X\sim F_{i^*}} [(X-(1-\epsilon)OPT)^+] \nonumber \\
        & \geq & p\cdot (1-\epsilon)OPT + \sum_{j=1}^k (1-p) \cdot R_{i^*} \stackrel{\eqref{eq:ri}}{\geq} p\cdot (1-\epsilon)OPT +  \frac{n}{\epsilon} \cdot (1-p) \frac{\epsilon \cdot OPT}{n} \nonumber \\ &\geq & (1-\epsilon)OPT, \nonumber
    \end{eqnarray}
    which concludes the proof.
\end{proof}

We next show that no algorithm can achieve better competition complexity asymptotically.
\begin{lemma}\label{lem:adver-neg}
    For every $n\geq 2$ and $\epsilon\in [0,0.5)$ there exists  an instance with $n$ rewards defined by a distribution $F$, such that for every $k < \frac{n-1}{3\epsilon}$, there exists an arrival order over the $n\cdot k$ rewards, such that for every algorithm $\ALG$: $$ \E[\ALG(v)] <  (1-\epsilon)\cdot \E[\max_i v_i].$$
\end{lemma}
\begin{proof}
    Let $M= \frac{n}{\epsilon}$.
    Consider the instance where $F_1$ is deterministically $1$ and $F_i$ for $i\geq 2$ is $\frac{3M^{i-1}\epsilon}{n-1}$ with probability $\frac{1}{M^{i-1}}$ and $0$ otherwise. The correlation is such that $v_i$ can only be non-zero if $v_{i-1}$ is non-zero.
    The value of the prophet for such an instance is equal to \begin{eqnarray}    
    \E_{v\sim F}[\max_{i} v_i ]  & = & 1 \cdot\left(1-\frac{1}{M}\right)+\sum_{i=2}^{n-1} \frac{3M^{i-1}\epsilon}{n-1} \left(\frac{1}{M^{i-1}}  -\frac{1}{M^{i}} \right) + \frac{3M^{n-1}\epsilon}{n-1}\cdot \frac{1}{M^{n-1}} \nonumber \\ & = & 1-\frac{1}{M} +\frac{n-2}{n-1} \cdot 3\epsilon \left(1-\frac{1}{M}\right) +\frac{3\epsilon}{n-1} > \frac{1}{1-\epsilon} .\label{eq:pro}
    \end{eqnarray}
    Now, consider an instance with $k$ copies of all the types, where the rewards arrive in an order according to their type (breaking ties arbitrarily).
    Since the order is fixed, then the optimal algorithm can be deterministic without loss of generality. The algorithm cannot select a value of $1$ since it is too low (because of Equation~\eqref{eq:pro}).
    It holds that for every type $i<n$, the expected value from accepting a copy of $i$ with a value of $\frac{3M^{i-1}\epsilon}{n-1}$ is equal to the expected value of discarding it and selecting the next reward of the same copy (that has a value of $\frac{3M^i\epsilon}{n-1}$ with a conditional probability of $\frac{1}{M}$). Thus, an optimal algorithm that does not select a value of $1$ is to reject all rewards of type $i<n$.
    Such an algorithm achieves an expected value of $$ \E[\ALG(v)] = \frac{3M^{n-1}\epsilon}{n-1} \left(1-\left(1-\frac{1}{M^{n-1}}\right)^k\right) \leq k\cdot \frac{3\epsilon}{n-1}.$$
    Thus, in order to get at least $(1-\epsilon)$ of $\E_{v\sim F}[\max_i v_i]$, $k$ must be at least $\frac{n-1}{3\epsilon}$.
\end{proof}
\section{Competition Complexity of the Pairwise Independent Case}
\label{sec:cci}
In this section, we present how to adapt our algorithm to the case of pairwise independent prophet inequality which leads to a simple algorithm (and analysis) with optimal asymptotic competition complexity. In particular, since the pairwise independent case generalizes the independent case considered in \cite{DBLP:journals/corr/abs-2402-11084} the following result presents an optimal asymptotic competition complexity for the independent case with a simpler analysis.
\begin{theorem}\label{thm:cc-pairwise}
Let $\xi = \frac{3-\sqrt{5}}{2}$.
For the pairwise independent prophet inequality case, Algorithm~\ref{alg:cap-ind} achieves a competition complexity $k(n,\epsilon)$ of $\lceil\log\log_{(1+\xi)}{\frac{1}{\varepsilon}}\rceil +4$.    
\end{theorem}

\begin{algorithm}
\caption{Competition Complexity of Pairwise Independent Prophet Inequality}\label{alg:cap-ind}
\begin{algorithmic}[1]
\Require{$\epsilon$ - target approximation}

\State Let $\xi = \frac{3-\sqrt{5}}{2}$
\State Let $T_0$ be a threshold such that  $ \sum_i \Pr[v_i > T_0] \leq  \xi \leq\sum_i \Pr[v_i \geq T_0]$ \label{step:t0}
\For{ round $t=1, 2$} \Comment{Phase 1}
\For{$i=1,\ldots, n$}
\State Let $A_{t,i} = 1$

\If{$v_i^{(t)} > T_0$}
\State Return $v_i^{(t)}$
\EndIf
\EndFor
\EndFor
\If{no value was selected}
\State Apply Algorithm~\ref{alg:cap-p2} with parameters $m_2= \lceil \log\log_{(1+\xi)}\frac{1}{\epsilon}\rceil +2$ and initial threshold $T_0$ \Comment{Phase 2} \label{step:call2}
\EndIf 
\end{algorithmic}
\end{algorithm}
For simplicity, we denote the value of  $\lceil \log\log_{(1+\xi)}{\frac{1}{\varepsilon}} \rceil  $ by $r$. Therefore, the algorithm runs for $r+4$ rounds, that is, the competition complexity is $k(n,\varepsilon) = r+4$ (i.e., it does not depend on $n$). 
In the first phase, we apply for $2$ rounds (instead of $n+1$ rounds) a threshold strategy (instead of the OCRS) with threshold $T_0$ (observe that $T_0$ of Algorithm~\ref{alg:cap-ind} is a different quantile threshold than the one used in Algorithm~\ref{alg:cap}).
We show that setting a threshold of $T_0$ in the pairwise independent case serves a similar purpose to the purpose of the OCRS and covers the contribution of values above 
$T_0$. The second phase covers the contributions of values between the initial threshold $T_0$ and the last threshold $T_r$.

Our analysis mimics the analysis of Algorithm~\ref{alg:cap}, in particular our goal is to show that 
\begin{equation}
    \Pr_{v\sim F}[\max_i v_i  < T_r] \leq \varepsilon , \label{eq:eps}
\end{equation} and that  for every value $y\geq T_r$, Inequality~\eqref{eq:sd} holds. 
To show Inequality~\eqref{eq:eps}, we use the following lemma of \citet{caragiannis2021relaxing}.
\begin{restatable}[\citep{caragiannis2021relaxing}]{lem}{caragiannisold}
\label{lem:caragi}
Let $E_1,\ldots,E_n$ be a set of random events. Let $F$ be the joint pairwise independent distribution of them. Let $p_i=\Pr[E_i]$. Then, $$ \Pr\left[\bigvee_{i} E_i\right] \geq \frac{\sum_{i=1}^n p_i}{1+\sum_{i=1}^n p_i}.$$
\end{restatable}

Let $p = \Pr[\max_j v_j < T_0 ]$ as defined in Step~\ref{step:p} in Algorithm~\ref{alg:cap-p2} when called in Step~\ref{step:call2} of Algorithm~\ref{alg:cap-ind}.
Then by the definition of $T_0$ in Step~\ref{step:t0} of Algorithm~\ref{alg:cap-ind} and by Lemma~\ref{lem:caragi} it holds that $$1-p = \Pr[\max_j v_j \geq T_0 ]\geq \frac{\sum_{j=1}^n \Pr[v_j\geq T_0]}{1+\sum_{j=1}^n \Pr[v_j \geq T_0]} \geq  \frac{\xi}{1+\xi}, $$
where in the last inequality we used that the function $\frac{x}{1+x}$ is monotonically increasing.
Therefore, $p \leq 1-\frac{\xi}{1+\xi} = \frac{1}{1+\xi}$
which by combining with the value of $r$, implies that 
$$ \Pr[ \max_i v_i < T_{r}]  \leq p^{2^r} \leq \left(\frac{1}{1+\xi}\right)^{2^r} \leq  \varepsilon.$$

To prove the guarantee of Algorithm~\ref{alg:cap-ind}, we need to prove Inequality~\eqref{eq:sd}  for every value of $y\geq T_r$.
We use the notation of $q$, $q_y$, from the proof of Section~\ref{sec:ccc} (now $q$ is defined with respect to the new value of $T_0$).
We again divide into the same cases, and we note that the proofs of the cases where $ T_r \leq y\leq T_0$  (cases 2, 3) still hold.

We thus assume in the remainder of this section that $y>T_0$, and that $q>0$.

To analyze the performance of Algorithm~\ref{alg:cap-ind}, we analyze the first $2$ rounds of the algorithm. The first step of our analysis follows the analysis of \cite{caragiannis2021relaxing}. 

We define $\mathcal{E}^{t}_{i,x}$ to be the event that $v_i^{(t)}=x;$ and for all $  j \neq i , v_j^{(t)} \leq T_0$. 
We next use the following lemma which was implicitly proved in \citep{caragiannis2021relaxing}. We prove the lemma in Appendix~\ref{app:missing} for completeness.

\begin{restatable}[\citep{caragiannis2021relaxing}]{lem}{caragiannis}
For every, $t\in \{1,2\}$, $i\in [n]$, and $x>T_0$, it holds that $$\Pr[\mathcal{E}^{t}_{i,x}] \geq \left(1-\xi\right)\cdot \Pr_{v\sim F}[v_i=x].$$ \label{lem:pairwise-reach}
\end{restatable}

We next prove Inequality~\eqref{eq:sd} for values $y$ satisfying $y>T_0$.
To do so, we note that if round $t$ for $t=1,2$ was reached ($A_{t,1}$) and some event $\mathcal{E}_{i,x}^t$ holds for some $i$ and for $x\geq y$, then our algorithm selects a value of at least $y$. Moreover, for each $t\in\{1,2\}$, the events $\{\mathcal{E}_{i,x}^t\}_{i\in[n],x \geq y}$ are disjoint since $y>T_0$. We also note that the probability that we reach the second round satisfies that 
$$ \Pr[A_{2,1}] = (1-q),$$ and that we have symmetry between the first two rounds. Thus
\begin{eqnarray*}
      \Pr[\ALG(v) \geq y] &\geq&  \sum_{t=1}^2\sum_{i=1}^n\sum_{x  \geq y} \Pr[\mathcal{E}^{t}_{i,x} \wedge A_{t,1}] 
      \\&=& \left(2-q\right) \sum_{i=1}^n\sum_{x \geq y} \Pr[\mathcal{E}^1_{i,x}]
      \\&\geq& \left(2-q\right)\left(1-\xi\right) \sum_{i=1}^n\sum_{x \geq y}\Pr_{v\sim F}[v_i=x]
      \\&\geq& \left(2-\xi\right)\left(1-\xi\right) q_y
      = q_y,
\end{eqnarray*}
where the first inequality is since if $\mathcal{E}^{t}_{i,x} \wedge A_{t,1}$ holds for $x \geq y$, we select this value, and those events are disjoint; the first equality is since the probability of reaching the second round is $1-q$, and by the symmetry between the rounds given that we reach them;  the second inequality is by Lemma~\ref{lem:pairwise-reach}; the third inequality is since $\xi \geq q$, and by the union bound; the last equality is by the value of $\xi$.

This concludes the proof of Inequality~\eqref{eq:sd} for values of $y>T_0$, which concludes the proof of Theorem~\ref{thm:cc-pairwise}.

\section{Hardness Results}
\label{sec:hardness}
In this section, we present two hardness results for the competition complexity of correlated prophet inequality.
First, we show that a linear dependence on $n$ is unavoidable. Combining this with the hardness result of \cite{DBLP:conf/sigecom/BrustleCDV22} implies that Theorem~\ref{thm:cc} is asymptotically tight.
\begin{proposition}
For every $0 \leq \varepsilon < 1$, and for every $n$
there is an instance of correlated prophet inequality with $n$ rewards for which the $(1-\varepsilon)$-competition complexity of any algorithm,  $k(n,\varepsilon)$ is  at least $ n - n\cdot \varepsilon-1$. \label{prop:hard1}
\end{proposition}
\begin{proof}
   Let $\delta= \frac{\varepsilon}{n}$.
   Consider an instance with $n$ rewards, where reward $i=1,\ldots,n$ is equal to $\frac{1}{\delta^i}$ with probability $\delta^i$ and $0$ otherwise. The correlation between the rewards is such that a reward can be non-zero only if all previous rewards are realized to be non-zero. In other words $v_1 =\frac{1}{\delta}$ with probability $\delta$, and for $i>1$, $v_i=\frac{1}{\delta^i}$ if $v_{i-1} = \frac{1}{\delta^{i-1}} $ with probability $\delta$ (and otherwise $v_i=0$).
   The value of the prophet in this instance is $$ \E[\max_i v_i] = \sum_{i=1}^{n-1} \frac{1}{\delta^i} \cdot (\delta^i - \delta^{i+1}) + \frac{1}{\delta^n} \cdot \delta^n = n-(n-1)\cdot \delta = n -\varepsilon + \frac{\varepsilon}{n} .$$
   On the other hand, on a single copy no online algorithm can receive an expected value of more than $1$. This can be shown by the standard backward induction argument on this instance, where one can observe that given any history of length at most $n-1$, the expected value of the algorithm if selecting the current value is equal to the expected value of the strategy that does not select the current value.  Thus, the expectation of every strategy is at most the expectation of the first reward which is 1.

   Now, an algorithm that uses a competition complexity of at most $k$ cannot achieve a value of more than $k$ times the value of a single copy.  Thus, in order to achieve a $(1-\varepsilon)$-competition complexity, $k(n,\varepsilon) \geq   (1-\varepsilon)\cdot (n-\varepsilon+\frac{\varepsilon}{n})   \geq  
 n\cdot (1-\varepsilon) -1$.  
\end{proof}

    The former proposition shows that the first term of our bound on the competition complexity (Theorem~\ref{alg:cap}) is not only tight asymptotically, but also tight in terms of the constant. Proposition~\ref{prop:hard1} basically shows that in order to cover the contribution of values in the $\epsilon$-upper quantile, one needs at least $n-o(n)$ rounds. This observation leads us to design the first phase of Algorithm~\ref{alg:cap}, which covers the contribution of the upper quantile in the first $n+1$ rounds.
    We next show, that also our second part of the analysis is tight up to the exact constant. In particular, once an algorithm covers the contribution of values in the upper quantile (which corresponds to the end of Phase~1 of Algorithm~\ref{alg:cap}), one needs $(1-o(1))\cdot \log\log\frac{1}{\epsilon}$ additional rounds in order to $(1-\epsilon)$-stochastically dominate $F^*$.

\begin{proposition}
    Let $\ALG$ be some algorithm that selects a value in the first $m_1$ rounds with probability at most $\frac{1}{2}$. Then in order for $\ALG$ to $(1-\epsilon)$-stochastically dominate $F^*$, $\ALG$ must use at least $m_1+ \log\log \frac{1}{4\epsilon}$ rounds in total.  
    \label{prop:tight}
\end{proposition}
\begin{proof}
    Consider an instance with $n>\frac{1}{\epsilon}$ rewards where the value of reward $i$ is $i$ with probability $\frac{n+1-i}{n}$ and $0$ otherwise, where a reward's value can be non-zero only if all previous rewards are non-zero.
    In other words, the first reward has a deterministic value of $1$, and reward $i>1$ has a value of $i$ with probability $\frac{n+1-i}{n+2-i}$ if the value of reward $i-1$ is $i-1$, and otherwise, reward $i$ is $0$.

    The distribution of the prophet on a single copy is a uniform distribution over $\{1,\ldots,n\}$.
    This is since for $j<n$ it holds that $$\Pr[\max_{i} v_i  =j] = \frac{n+1-j}{n} \cdot (1-\frac{n-j}{n+1-j}) = \frac{1}{n} ,$$
    and for $j=n$, $ \Pr[\max_{i} v_i  =j] = \frac{n+1-j}{n} = \frac{1}{n}$.

    Consider an algorithm $\ALG$ on $m_1+ m_2$ blocks that $(1-\epsilon)$-stochastically dominates the performance of the prophet.
    We can assume without loss of generality that the decisions of the algorithm given that the algorithm reached block $i$  are independent of the realizations of previous blocks.
       
    For $i\geq 1$, let $p_i$ be the probability that $\ALG$ selects a value during block $m_1+i$ given this block is reached.  We will assume that for block $m_1+ i$, the algorithm uses a threshold $T_i$ for which $\Pr_{v\sim F}[\max_i v_i \geq T_i] \geq p_i \geq \Pr_{v\sim F}[\max_i v_i > T_i] $, where the value $T_i$ is selected with probability $ s$ such that $ \Pr_{v\sim F}[\max_i v_i > T_i] + \Pr_{v\sim F}[\max_i v_i = T_i] \cdot s =p_i$. This threshold strategy for block $m_1+i$ stochastically dominates any other strategy that selects a reward with probability $p_i$.
We can assume without loss of generality that $p_i$ is an increasing sequence; otherwise, an algorithm with sorted $p_i$'s stochastically dominates any algorithm with the same $p_i$'s.

    For $i\geq 0$, let $q_i $ be the probability that $\ALG$ does not select a value up to block $m_1+i$. By the assumption of the proposition, it holds that $q_0\geq \frac{1}{2}$. For $i\geq 1$, $q_i = q_{i-1} \cdot (1- p_i)$.

    We next prove that for $i\geq 1$ it holds that $$p_i \leq  1- q_{i-1} +3\epsilon.$$
    Assume towards contradiction that exists $i$ for which  $p_i >  1- q_{i-1} +3\epsilon $. Let $i^*$ be the minimal such $i$. We note that $\ALG$ in block $j\geq m_1+i^*$ will only select values that are strictly below $T_{i^*}+2$ (as if a value which is at least $T_{i^*}+2$ appears, then so a value that is at least $T_{i^*}+1$ and less than $T_{i^*}+2$ appear before). Thus,   $$ \Pr[\ALG \geq T_{i^*}+2 ] \leq  1- q_{i^*-1} .$$
    On the other hand, by definition of $T_{i^*}$ and since for every value of $x$, the probability that $\max_{i} v_i $ is in the interval $ [x,x+1) $ is at most $\frac{1}{n} $ which is at most $\epsilon $, we get that:
    $$ \Pr[\max_{i} v_i \geq T_{i^*}+2 ]  =\Pr[\max_{i} v_i \geq T_{i^*} ] - \Pr[\max_{i} v_i \in [ T_{i^*}, T_{i^*}+2 )  ] \geq p_{i^*}- 2\epsilon>   1- q_{i^*-1} + \epsilon,$$
    which contradicts that $\ALG$, $(1-\epsilon)$-stochastically dominates $F^*$.

    Thus, we get that $$ q_i \geq q_{i-1} \cdot \left( q_{i-1} -3\epsilon \right) . $$

    We next prove by induction that $ q_i \geq  2^{-2^i} - 3\epsilon$.
    For $i=0$, it holds by our assumption that $q_0 \geq \frac{1}{2}$.
     For $i=1$, it holds since $$ q_1 \geq q_0 \cdot (q_0-3\epsilon) \geq  \frac{1}{2} \left(\frac{1}{2}-3\epsilon\right) \geq 2^{-2^1}  -3\epsilon .$$
    For $i\geq 1$, then $$ q_{i+1} \geq q_{i} \cdot \left( q_{i} -3\epsilon \right) \geq \left(2^{-2^i} - 3\epsilon  \right) \cdot \left( 2^{-2^i} - 6\epsilon \right) \geq 2^{-2^{i+1}} -9\epsilon \cdot 2^{-2^i}  \geq 2^{-2^{i+1}} -3\epsilon. $$

    Thus, in order for $\ALG$ to $(1-\epsilon)$-stocahstically dominate $F^*$, the algorithm must selects overall with probability at least $(1-\epsilon)$, thus $q_{m_2} \leq \epsilon$, which implies that $$2^{-2^{m_2}} -3\epsilon \leq \epsilon, $$
    which implies that $ m_2 \geq \log \left( \log(\frac{1}{4\epsilon})\right),$ that concludes the proof.    
 \end{proof}

\begin{remark}
    We note that although both parts of the bound of Theorem~\ref{thm:cc} are tight up to the constant, it might be possible to improve the constants for the overall expression for cases where $n$ is of an order of $\log\log\frac{1}{\epsilon}$. 
\end{remark}

We next show that in contrast to the case of independent prophet inequality, where \cite{DBLP:journals/corr/abs-2402-11084} showed that the $(1-\varepsilon)$-competition complexity of block threshold algorithms\footnote{A block threshold algorithm is an algorithm that sets a threshold for each copy (that might be different for different copies).  \cite{DBLP:journals/corr/abs-2402-11084} showed that for the independent case, block threshold algorithms achieve the optimal competition complexity, which reduced the problem to writing an optimization problem as a function of the quantiles of the block thresholds, and approximately solving this optimization.} is $O\left(\log\log\frac{1}{\varepsilon}\right)$, in the case of correlated prophet inequality, the competition complexity of block threshold algorithms can be unbounded. 
Moreover, this claim remains true even when the number of rewards is constant.
\begin{proposition}
For every $\varepsilon<1$ and $M>0$ there is an instance of correlated prophet inequality with $n=2$ rewards for which for every block threshold algorithm, the $(1-\varepsilon)$-competition complexity is at least $M$. \label{prop:hard2}
\end{proposition}
\begin{proof}
    Let $\xi=\lceil\frac{2M}{1-\varepsilon}\rceil$, and let $X$ be a random variable that is distributed according to $G\left(1-\frac{1}{\xi}\right)$.
    Consider an instance of the correlated prophet inequality with two rewards where $v_1 = \xi^{\min(X,\xi)}$, and $v_2 = \xi^{\min(X,\xi)+1}$.
The value of the prophet is:
    \begin{eqnarray}
    \E[\max(v_1,v_2)]  & = & \E[v_2]  \nonumber \\ & = & \Pr[X\geq \xi] \cdot \xi^{\xi+1}+\sum_{i=1}^{\xi-1}\Pr[X =i] \cdot \xi^{i+1}  \nonumber \\ & = & \left(\frac{1}{\xi}\right)^{ \xi-1}  \cdot \xi^{\xi+1} +  \sum_{i=1}^{\xi-1} \left(\frac{1}{\xi}\right)^{i-1} (1-\frac{1}{\xi}) \cdot \xi^{i+1} \nonumber \\ & = & \xi^3-\xi^2+\xi. \label{eq:prophetxi}
    \end{eqnarray}

We next analyze the performance of the single threshold algorithm with threshold $\tau$ on a single copy. We assume without loss of generality that $\tau=\xi^{i^*}$ for some $i^*\in \{1,\ldots,\xi+1\}$, then we can bound the performance of the algorithm by the following derivations:
If $i^*\leq \xi $ then 
    \begin{eqnarray*}
        \E[\ALG] & \leq  & \sum_{i=i^*}^{\xi-1}\Pr[X = i] \cdot \xi^i + \Pr[X \geq \xi]\cdot \xi^{\xi} + \Pr[X =i^*-1] \cdot \xi^{i^*} \\ & \leq & \sum_{i=i^*}^{\xi-1}\left(\frac{1}{\xi}\right)^{i-1} \cdot (1-\frac{1}{\xi}) \cdot \xi^i + \left(\frac{1}{\xi}\right)^{\xi-1}\cdot \xi^{\xi} + \left(\frac{1}{\xi}\right)^{i^*-2} \cdot (1-\frac{1}{\xi})\cdot \xi^{i^*} <  2 \xi^2 -2\xi+ 2, 
    \end{eqnarray*}
    where the last inequality follows since $i^*\geq 1$.
Else ($i^*=\xi+1$), then
    $$\E[\ALG] \leq   \Pr[X \geq \xi] \cdot \xi^{\xi+1} = \left(\frac{1}{\xi}\right)^{\xi-1} \cdot \xi^{\xi+1}  = \xi^2 < 2\xi^2-2\xi+2, $$
    where the last inequality holds for every value of $\xi$.

Combining this with Equation~\eqref{eq:prophetxi}, since the performance of an algorithm over $k$ copies can be at most $k$ times the performance on one copy, we get that $$ k \geq \frac{(1-\varepsilon)\cdot\E[\max(v_1,v_2)]}{\E[\ALG]} > \frac{(1-\varepsilon)\cdot  (\xi^3-\xi^2+\xi) }{ 2 \xi^2 -2\xi+ 2} =   \frac{(1-\varepsilon)\cdot \xi}{2} \geq M , $$
which concludes the proof.    
\end{proof}

\bibliography{bib}
\bibliographystyle{abbrvnat}
\appendix

\section{Missing Proofs}
\label{app:missing}
\caragiannis*
\begin{proof}
    By definition $$\Pr[\mathcal{E}^{t}_{i,x}] = \Pr_{v \sim F}[v_i^{(t)}=x] \cdot \Pr_{v^{(t)} \sim F}[\forall {j \neq i}:~ v_j^{(t)} \leq T_0 \mid v_i^{(t)}=x].$$
     By the union bound it holds that $$\Pr_{v^{(t)} \sim F}[\forall {j \neq i}:~ v_j^{(t)} \leq T_0 \mid v_i^{(t)}=x] \geq 1 - \sum_{j \neq i} \Pr_{v^{(t)} \sim F}[v_j^{(t)} > T_0 \mid v_i^{(t)}=x].$$
     Due to pairwise independence it holds that $$\Pr_{v^{(t)} \sim F}[v_j^{(t)} > T_0 \mid v_i^{(t)}=x] = \Pr_{v^{(t)} \sim F}[v_j^{(t)} > T_0].$$
     By definition of $\xi$ it holds that $$\sum_{j \neq i} \Pr_{v^{(t)} \sim F}[v_j^{(t)} > T_0] \leq \xi.$$
     Combining everything together we get that: $$\Pr[\mathcal{E}^{t}_{i,x}] \geq \Pr_{v^{(t)} \sim F}[v_i^{(t)}=x] \cdot \left( 1 - \sum_{j \neq i} \Pr_{v^{(t)} \sim F}[v_j^{(t)} > T_0] \right) \geq \left( 1 - \xi \right) \cdot \Pr_{v^{(t)} \sim F}[v_i^{(t)}=x] =\left( 1 - \xi \right) \cdot \Pr_{v \sim F}[v_i=x] ,$$
     which concludes the proof.
\end{proof}
\end{document}